\documentclass[final,3p,times]{elsarticle}
\usepackage{graphics}
\usepackage{subfigure}
\usepackage{epstopdf}
\usepackage[figuresright]{rotating}
\usepackage{multicol}
\usepackage{color}
\usepackage{amsmath, amssymb}
\usepackage{amssymb}
\usepackage{amsthm}
\usepackage[ruled, linesnumbered]{algorithm2e}
\usepackage{algorithmic}
\usepackage{setspace}
\usepackage{multirow}
\usepackage{lineno}
\usepackage{array}
\usepackage{booktabs}
\usepackage{amsmath}
\newtheorem{theorem}{Theorem}
\newtheorem{lemma}{Lemma}
\usepackage{algorithm2e,setspace}
\graphicspath{{kbsGraphs/}}
\journal{KNOWLEDGE-BASED SYSTEMS}
\begin{document}
\begin{sloppypar}
\begin{frontmatter}
\title{Semi-supervised representation learning via dual autoencoders for domain adaptation }
\author[1,2]{Shuai Yang}
\ead{yangs@mail.hfut.edu.cn}
\author[1,2]{Hao Wang}
\ead{jsjxwangh@hfut.edu.cn}
\author[1,2]{Yuhong Zhang \corref{cor1}}
\cortext[cor1]{Corresponding author: Yuhong Zhang}
\ead{zhangyh@hfut.edu.cn}
\author[1,2]{Peipei Li}
\ead{peipeili@hfut.edu.cn}
\author[3]{Yi Zhu}
\ead{z8d1177@126.com}
\author[1,2,4]{Xuegang Hu}
\ead{jsjxhuxg@hfut.edu.cn}

%\address[rvt]{School of Computer Science and Information Engineering, Hefei University of Technology, Hefei, Anhui 23009, China}
%\address[three]{School of information Engineering, Yangzhou University, Yangzhou, Jiangsu 225009, China}
%\address[second]{Anhui Province Key Laboratory of Industry Safety and Emergency Technology, Hefei, Anhui 23009, China}
\address[1]{ Key Laboratory of Knowledge Engineering with Big Data (Hefei University of Technology), Ministry of Education}
\address[2]{ School of Computer Science and Information Engineering, Hefei University of Technology, Hefei 230009, China}
\address[3]{ School of Information Engineering, Yangzhou University, Jiangsu 225009, China}
\address[4]{ Anhui Province Key Laboratory of Industry Safety and Emergency Technology, Hefei 230009, China}

\begin{abstract}
Domain adaptation aims to exploit the knowledge in source domain to promote the learning tasks in target domain, which plays a critical role in real-world applications. Recently, lots of deep learning approaches based on autoencoders have achieved  a significance performance in domain adaptation. However,  most existing methods focus on minimizing the distribution divergence by  putting the source  and target data together to learn global feature representations, while they do not consider the local relationship between instances in the same category from different domains. To address this problem, we propose a novel Semi-Supervised Representation Learning framework via Dual Autoencoders for domain adaptation, named SSRLDA. More specifically, we extract richer  feature representations by learning the global and local feature representations simultaneously using two novel autoencoders, which are referred to as  marginalized denoising autoencoder with adaptation distribution (MDA$_{ad}$) and  multi-class marginalized denoising autoencoder (MMDA) respectively. Meanwhile, we  make full use of  label information to optimize  feature representations. Experimental results show that  our proposed approach  outperforms several state-of-the-art baseline methods.
\end{abstract}
\begin{keyword}
domain adaptation;  dual autoencoders; representation learning;  semi-supervised.
\end{keyword}
\end{frontmatter}

\section{Introduction}
Traditional  classification methods built on the assumption that the training data  and testing  data are identically distributed, and thus the classifier trained from training data can be used to classify the testing data directly. However, in many real-world applications, the distributions of training and testing data are usually different but related \citep{2,WeiKG17,PouyanfarSYTTRS19,YangLZ17,JiangGLLCH19}. To address this problem,  a number of domain adaptation approaches \citep{PanY10,SaghaCS17,TanSKZYL18,ZellingerMGLNS19} have been  proposed to learn the domain invariant feature representations on which the  divergence of the domain   distribution was decreased.
\par In  recent years, deep learning has been  widely used in natural language processing and achieved a satisfying classification performance.  Popular deep learning models included Autoencoder \citep{16,17,MengD0018,WeiFeature}, Convolutional Neural Network (CNN) \citep{31,EyiokurYE18,GaoSX18}, Recurrent Neural Network (RNN) \citep{JaechHO16,DingYJ17,DengZS18} and Generative Adversarial Network (GAN) \citep{TzengHSD17,LongC0J18,CaoMLW18,ShenQZY18}. Methods based on autoencoders have been proven to be able to learn domain generic concepts and  be beneficial  to cross-domain learning tasks.
\par A number of existing domain adaptation approaches based on autoencoders  have been successfully applied in learning powerful feature representations. One typical unsupervised feature representation learning method is stacked denoising autoencoder (SDA) \cite{16}, which learned  deep feature  representations by stacking multi-layer denoising autoencoder. However, SDA has two  shortcomings: high computational cost and the lack of scalability to high-dimensional features. To address these problems,  Chen et al. \cite{17} put forward marginalized stacked denoising autoencoder (mSDA). It obtained higher-level feature representations by marginalization corrupting the original input data, and it is a more efficient method than SDA. Nevertheless, mSDA has the following two  issues: one is how to determine the noise probability, and the other is that some features may be more important than the others in domain adaptation. Based on the work of mSDA, Wei et al. \cite{WeiFeature}  proposed the feature analysis of marginalized stacked denoising autoenconder (DTFC), which extracted  effective feature representations by corrupting the raw input data with multinomial dropout noise. Futhermore,  in order to capture much more nonlinear relationship in the data, deep nonlinear feature coding (DNFC) \cite{WeiKG16} minimized the marginal distribution between source  and target domains and introduced kernelization for nonlinear coding. To avoid over-fitting on source training data,  Csurka et al. \cite{Csurka17} put forward an extended framework for domain adaptation, which learned better feature representations by minimizing the domain prediction loss and the maximum mean discrepancy between source domain and target domain. In addition, Yang et al. \cite{YangZZLH19} proposed a  representation learning framework based on serial autoencoders (SEAE), which  combined two different types of autoencoders in series to achieve richer feature representations. Additionally,  Jiang et al. \cite{JiangGLLCH19} proposed a method of stacked robust adaptively regularized auto-regressions, which adopted $\ell_{2,1}$-norm  as a measure of reconstruction error to reduce the effect of outliers. Based on the work of Jiang \cite{JiangGLLCH19}, $\ell_{2,1}$-norm stacked adaptation regularization autoencoder (SRAAR) \cite{YangZZDH19} was presented to learn more complicated feature representations by preserving the local geometry structure of data in source and target domains.

\par Though previous deep learning approaches based on autoencoders are able to reduce the distribution discrepancy between two  domains with  the learned  feature representations, there are two major drawbacks as follows: firstly, most of approaches only focus on learning global feature representations by putting source and target data together for training, and pay  little attention to  the local relationship between instances in the same category from two domains. Thus, they probably ignore some discriminative features  that are beneficial to cross-domain tasks.  For example, some discriminative features belonging to one category  will be disturbed by other features belonging to other categories if only the global feature representations are  considered. Secondly, the label information of source domain can be well used to further reduce the condition distribution divergence between two domains, however, most of approaches are unsupervised models and they can not use the label information of source domain to improve the quality of feature representations, which are easy to result in an unsatisfying performance.

%\par Though deep learning methods based on autoencoders are able to learn abstract feature representations, there are two  main limitations as follow: firstly, most of works only focus on learning global feature representations by putting source and target data together for training, and pay a little attention to  the local relationship between instances of the same category in two domains, which probably ignore some discriminative features  that are benefit for cross-domain tasks. Secondly, most of approaches are unsupervised model and only use one type of autoencoder  to learn new representations, which are easy to result in over-fitting and  pose challenges for learning different characteristics of data.

To address the above issues, we propose a novel Semi-Supervised  Representation Learning framework via Dual Autoencoders (SSRLDA) for domain adaptation. More specifically, marginalized denoising autoencoder with adaptation distributions (MDA$_{ad}$) is used to obtain the global feature representations  by minimizing the marginal and condition distributions between two domains simultaneously. Second, we  divide the source and target data into different local subsets, which  depend on the category of domains by utilizing the label knowledge in  source domain and the pseudo label knowledge in  target domain. Then, multi-class marginalized denoising autoencoder (MMDA) is applied in learning feature representations of the local subsets. Finally, we combine the global and local feature representations to form the new feature space, on which the distribution divergence is decreased. Then, we conduct cross-domain learning tasks on the new feature space.
\par  Our main contributions of SSRLDA are summarized below:
\begin{itemize}
\item Different from the existing works, which focus on learning the global feature representations for  both source and target domains, we provide a novel viewpoint for solving domain adaptation problems by taking both global and local feature representations into consideration. In this way, the proposed algorithm can extract richer feature representations with different characteristics of the data and further improve the classification accuracy.
\item Without the label information of target domain, we design a new semi-supervised feature representation learning framework that makes full use of the pseudo label knowledge of target domain to optimize feature representations. More specifically, the proposed algorithm enhances the quality of feature representations by minimizing the conditional distribution between source and target domains with the help of label information of source domain and the pseudo label information of target domain.
\item We design two novel autoencoder models to learn  robust feature representations,  which are referred to as marginalized denoising autoencoder with adaptation distributions (MDA$_{ad}$) and  multi-class marginalized denoising autoencoder (MMDA) respectively. Both of them can obtain better feature representations. In addition, both of them can be calculated in a closed form with less time cost.
\end{itemize}

\section{Related Work}
Domain adaptation has received lots of attention from researches in  recent years. How to reduce the distribution divergence between source domain and target domain is the core idea of domain adaptation.  A number of  feature-based approaches have been presented to address this problem.  Those approaches aimed to extract common feature representations shared by source and target domains, which are  beneficial to  reducing the distribution divergence between two domains. Those approaches  can be categorized into shallow learning methods and deep learning methods regarding the adopted techniques.

\subsection{Shallow Learning Methods}
One typical shallow learning method is  a structural correspondence learning (SCL) \cite{BlitzerMP06}, which  used the pivot features between source and target domains to extract the  potential relationships  between non-pivot and pivot features.  Pan et al. \cite{PanTKY11} put forward a  subspace learning method by utilizing maximum mean discrepancy (MMD) to learn a subspace on which  the domain divergence between two domains is reduced. Li et al. \cite{LiWL17} designed a powerful and  efficient algorithm, which  aimed at seeking a discriminate subspace shared by source and target domains.  CORAL \cite{SunFS16} decreased the distance between source and target domains by  aligning the second-order statistics of two domains distributions. It is an effective and simple feature learning method.  Sharma et al. \cite{BhattacharyyaDS18}  designed a  method of identifying transferable information across domains, it used the word with the same polarity  between two domains for sentiment classification. Di et al. \cite{DiPSC18} proposed  a  transfer learning framework named transfer learning via feature isomorphism discovery (TLFid), which used extra feature labeling information and  the feature isomorphism across domains  to improve the performance of knowledge transfer. Luo et al. \cite{LuoWLT19} proposed a general heterogeneous transfer distance metric learning framework, which utilized the knowledge fragments in  source domain to aid to the metric learning in  target domain. Chen et al. \cite{ChenSLYW19} presented an unsupervised domain adaptation algorithm called domain space transfer extreme learning machine (DST-ELM), DST-ELM  incorporated MMD  into the extreme learning machine (ELM) framework to preserve  the discriminative information of target domain. He et al. \cite{HeTL19} designed a  method called multi-view transfer learning with privileged learning framework, which performed better in both multi-view and cross-domain learning tasks. Liu et al. \cite{LiuLZ18} designed a  heterogeneous unsupervised domain adaptation model based on n-dimensional fuzzy geometry and fuzzy equivalence relations (F-HeUDA). F-HeUDA ont only  effectively transferred knowledge from large datasets to small datasets, but also allowed  that source and target domains have different numbers of instances.
\subsection{Deep Learning Methods}
Deep learning methods  are considered to be an effective way to learn robust feature representations for domain adaptation and have attracted more and more attention. For instances, Glorot et al. \cite{16} proposed stacked denoising autoencoder (SDA) to learn  deep  feature representations.   Based the work of SDA, Chen et al.  \cite{17} put foward marginalized stacked denoising autoencoder (mSDA), which adopted the liner denoiser to learn transferable feature representations and did not need the stochastic gradient descent to learn parameters. Wei et al. \cite{WeiKG16} designed a deep nonlinear feature coding (DNFC) framework, which incorporated kernelization  and MMD into mSDA to learn nonlinear deep feature representations. Feature analysis of marginalized stacked denoising autoenconder (DTFC) \cite{WeiFeature} obtained deep feature representations by corrupting the original input data with  multinomial dropout noise.  Jiang et al. \cite{33} adopted $\ell_{2,1}$-norm to measure the reconstruction error to reduce the impact of outliers.  Ganin et al. \cite{19}  presented domain-adversarial training of neural networks (DANN), which  learned domain invariant features by a minimax game  between the domain classifier and the feature extractor.  Based on the work of  Ganin \citep{19},  Clinchant et al. \cite{20} proposed an unsupervised regularization transfer learning method  to  avoid overfitting. Zhu et al. \cite{Zhu18} presented  a semi-supervised method for transfer learning called stacked reconstruction independent component analysis (SRICA), which minimized the KL-Divergence between two domains and adopted the softmax regression to encode label information of the data in source domain. Shen et al. \cite{ShenQZY18} presented  wasserstein distance guided representations learning (WDGRL),  which learned domain invariant feature representations in an adversarial manner by minimizing the wasserstein distance between two domains. Long et al. \cite{LongC0J18} designed an  approach of transfer learning with multimodal distributions,  named  conditional domain adversarial networks (CDANs).
\section{Preliminaries}
In this section, some notations and  preliminary knowledge used in our proposed framework are introduced.
\subsection{Notations }
 Given a labeled source domain $X_{S}=\{x_{i}^{s}\}_{i=1}^{n_{s}}$ with  label knowledge $Y_{S}=\{y_{i}^{s}\}_{i=1}^{n_{s}}$  and an unlabeled target domain $X_{T}=\{x_{j}^{t}\}_{i=1}^{n_{t}}$.  $n_{s}$ and $n_{t}$ are numbers of source and target instances respectively.  $x_{i}^{s} \in \Re^{d}$ and $x_{j}^{t} \in \Re^{d}$ are the $i$-th, $j$-th instances in  source domain and target domain respectively. $d$ is the dimension of the feature space, and  $y_{i}^{s}$ represents the label of instances in  source domain $X_{S}$.

\subsection{Marginalized Denoising Autoencoder }
Marginalized denoising autoencoder (MDA) \cite{17} aimed at learning domain invariant feature representations by reconstructing the original data from the marginalization corrupted ones. Given an input data $X\in  \Re^{(n_{s}+n_{t})\times d}$ and  $X$ is corrupted by random feature removal  with the  probability $p$. The $i$-th corrupted version of $X$ is defined as $\widetilde{X_{i}}$. MDA learns robust feature representations by minimizing the objective function, as shown in Eq.(\ref{LossMDA}):
\begin{equation}
\mathcal{L}(W)=\min\limits_{W}{\frac{1}{K}\sum_{i=1}^{K}\|X-\widetilde{X_{i}}W\|_{F}^{2}} + \lambda\|W\|_{F}^{2}
\label{LossMDA}
\end{equation}
where $\lambda$ is hyper-parameter. The objective function has a unique and optimal solution \cite{17}, and the mapping $W\in  \Re^{d\times d}$ can be expressed in a closed form  as $W= (Q+\lambda I_{d})^{-1}P$ with $Q=\frac{1}{K}\sum_{i=1}^{K}\widetilde{X_{i}}^{T}\widetilde{X_{i}}$ and $P=\frac{1}{K}\sum_{i=1}^{K}\widetilde{X_{i}}X^{T}$.  Chen et al.  \cite{17} showed the limit case ($K\to\infty$) and derived  the expectations of $Q$ and $P$, and $W$ can be computed as $W= (\mathbb{E}[Q]+\lambda I_{d})^{-1}\mathbb{E}[P]$ ($I_{d}$ is a diagonal matrix).

The diagonal entries of the matrix $\widetilde{X_{i}}^{T}\widetilde{X_{i}}$  hold with the probability $1-p$, and $\mathbb{E}[P]$ and $\mathbb{E}[Q]$ is shown in Eq.(\ref{eqmDAP}) and Eq.(\ref{eqmDAQ}) respectively:
\begin{equation}
  \mathbb{E}[P] =(1-p)U
\label{eqmDAP}
\end{equation}

\begin{equation}
\mathbb{E}[Q_{ij}] = \left\{
             \begin{array}{lcl}
             { U_{ij}(1-p)^{2}}  &, if \quad i\neq j \\
             {U_{ij}(1-p)}  &,if\quad  i = j
             \end{array}
        \right.
\label{eqmDAQ}
\end{equation}
where $U = X^{T}X$ is the covariance matrix of the uncorrupted data $X$. $U_{ij}$ and $Q_{ij}$ represent the $i$-th row, $j$-th column element in matrix $U$ and $Q$ respectively.

\subsection{Maximum Mean Discrepancy}
\par Maximum mean discrepancy (MMD) is widely used to measure the distribution divergence between source and target domains since it avoids the density estimation.  Most previous works  use MMD as a regularizer for  domain adaptation learning tasks \cite{Csurka17}. Given two sets of samples $X_{S}$ and $X_{T}$, the empirical estimate of MMD is defined as Eq.(\ref{eqOMMD}):
\begin{equation}
MMD^{2}(X_{S},X_{T})=\bigg\lVert\frac{1}{n_{s}} \sum_{i=1}^{n_{s}}\Phi(x_{i}^{s})-\frac{1}{n_{t}} \sum_{j=1}^{n_{t}}\Phi(x_{j}^{t})\bigg\lVert_{\mathcal{H}}^{2}
\label{eqOMMD}
\end{equation}
where $\Phi(x)$ is a feature map kernel function and $\mathcal{H}$ is the reproducing kernel Hilbert space (RKHS).

\section{Proposed Algorithm}
In this section, we first give some basic concepts  used in this paper, and then give the details of our SSRLDA algorithm.
\subsection{Problem Formalization }
 Given a labeled source domain $X_{S}=\{x_{i}^{s}\}_{i=1}^{n_{s}}$ and an unlabeled target domain $X_{T}=\{x_{j}^{t}\}_{i=1}^{n_{t}}$, the goal of our method is to learn effective and transferable feature representations, where a classifier can be trained to predict the labels $Y_{T}=\{y_{i}^{t}\}_{i=1}^{n_{t}}$ of instances in target domain $X_{T}$.

\subsection{Dual  Feature Representation Learning Framework (SSRLDA)}
The proposed semi-supervised representation learning  framework based on autoencoders (SSRLDA) is capable of learning more powerful and richer feature representations.  Our SSRLDA  combined dual feature representations learned from two types of  autoencoders which are referred to as marginalized denoising autoencoder with adaptation distributions (MDA$_{ad}$) and  multi-class  marginalized denoising autoencoder (MMDA) respectively. MDA$_{ad}$ is adopted to learn the global feature representations of the source and target domains by minimizing the marginal and condition distributions of two domains simultaneously, while MMDA is applied in extracting  the local feature representations. Additionally, we introduce label information to optimize feature representations. The global and local feature representations are combined to form dual feature representations, on which the distance between two domains is reduced. Fig.\ref{SRAEFframe} shows the framework of our proposed approach. In Fig.\ref{SRAEFframe}, $|C|$ represents the number of class in both domains, $l$ represents the number of staked layers, $X_{S}^{(c)}$ and $X_{T}^{(c)}$ $(c=1,\ldots,|C|)$ are  instances belonging to the $c$-th class in source domain and target domain respectively, $G(l)$ represents the $l$-th layer feature representations of data $G$ ($G$ can be $  X_{S},X_{T},X_{S}^{(c)},X_{T}^{(c)})$. In the following, we will give the details of MDA$_{ad}$ and MMDA.

\begin{figure*}
  \centering
  \includegraphics[width=13cm]{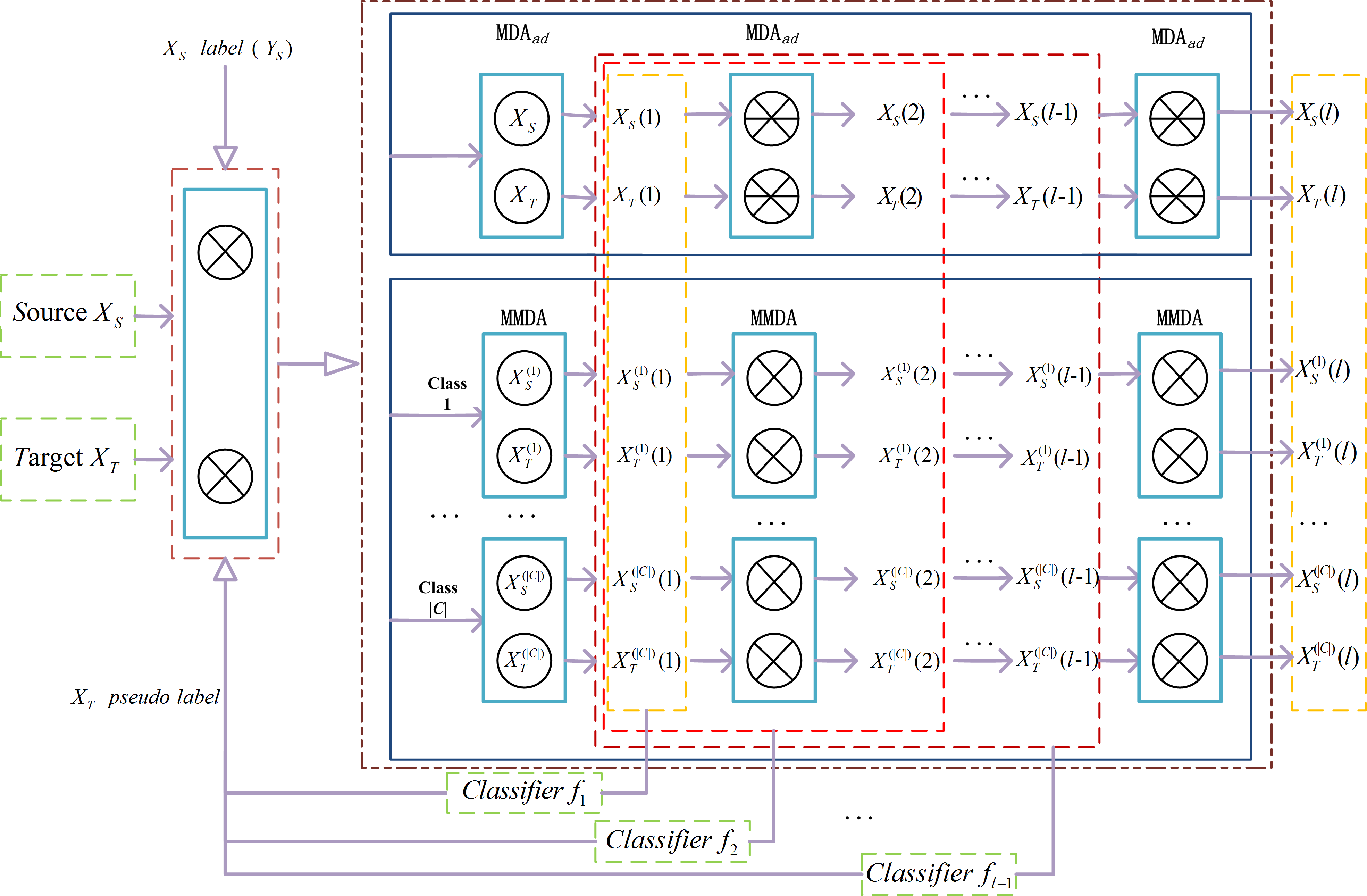}
  \caption{The whole framework of our proposed SSRLDA}
  \label{SRAEFframe}
\end{figure*}

\subsection{Marginalized Denoising Autoencoder with Adaptation Distributions (MDA$_{ad}$)}
\par Marginalized denoising autoencoder with adaptation distributions (MDA$_{ad}$) is used to capture global feature information in  source and target domains. MDA$_{ad}$ leverages both source data $X_{S}$ and target data $X_{T}$   to learn transferable feature representations. Generally, the distributions between source domain and target domain are different, i.e., $\mathcal{P}_{X_{S}}\neq$  $\mathcal{P}_{X_{T}}$  or  $\mathcal{Q}_{X_{S}}(Y_{S}|X_{S})\neq$ $\mathcal{Q}_{X_{T}}(Y_{T}|X_{T})$, where $\mathcal{P}_{X_{S}}$ and $\mathcal{P}_{X_{T}}$ are the marginal probability distributions of the source and target domains respectively, and $\mathcal{Q}_{X_{S}}(Y_{S}|X_{S})$ and $\mathcal{Q}_{X_{T}}(Y_{T}|X_{T})$ are the conditional distributions.  Based on the work of MDA, we minimize the marginal distributions and conditional distributions between two domains simultaneously in learning the global feature representations. More specifically, we learn the feature mapping matrix $W$ to reconstruct source domain and target domain. The input data is corrupted  with the noise probability $p$. The objective function is defined as Eq.(\ref{ProveeqmDAad}):
\begin{equation}
\mathcal{L}_{MDA_{ad}}(W)=\min\limits_{W}{\frac{1}{K}\sum_{i=1}^{K}\|X-\widetilde{X_{i}}W\|_{F}^{2}} + \lambda\|W\|_{F}^{2} + \beta (MMD^{2}_{mar}(\mathcal{P}_{\ddot{X}_{S}},\mathcal{P}_{\ddot{X}_{T}})+MMD^{2}_{con}(\mathcal{Q}_{\ddot{X}_{S}},\mathcal{Q}_{\ddot{X}_{T}}))
\label{ProveeqmDAad}
\end{equation}
where $\ddot{X}_{S}$ and $\ddot{X}_{T}$ are the induced data on a new feature space, $\lambda$ and $\beta$ are hyper-parameters and $MMD^{2}_{mar}(\mathcal{P}_{\ddot{X}_{S}},\mathcal{P}_{\ddot{X}_{T}})$ is defined as Eq.(\ref{eqMMDmar1}):

\begin{equation}
MMD^{2}_{mar}(\mathcal{P}_{\ddot{X}_{S}},\mathcal{P}_{\ddot{X}_{T}})=\bigg\|\frac{1}{n_{s}} \sum_{i=1}^{n_{s}}\widetilde{x_{i}^{s}}W-\frac{1}{n_{t}} \sum_{j=1}^{n_{t}}\widetilde{x_{j}^{t}}W\bigg\|_{F}^{2}=tr(W^{T}\widetilde{X}^{T}M_{0}\widetilde{X}W)
\label{eqMMDmar1}
\end{equation}
where $\widetilde{X} = \frac{1}{K}\sum_{i=1}^{K}\widetilde{X_{i}}$ and $M_{0}$ is computed as Eq.(\ref{eqm0}):
\begin{equation}
(M_{0})_{ij} = \left\{
             \begin{array}{ll}
             { \frac{1}{n_{s}n_{s}}} & if \quad x_{i},x_{j} \in X_{S}\\
             { \frac{1}{n_{t}n_{t}}} & if \quad x_{i},x_{j} \in X_{T}\\
             {\frac{-1}{n_{s}n_{t}}} & if \quad x_{i}\in X_{S},x_{j}\in X_{T}\\
             {\frac{-1}{n_{t}n_{s}}} & if \quad x_{i}\in X_{T},x_{j}\in X_{S}\\
             {0,} & otherwise
             \end{array}
        \right.
\label{eqm0}
\end{equation}
\par In order to minimize the condition distributions  between two  domains, the label knowledge of instances in source domain and the pseudo label knowledge of instances in target domain are leveraged.  Because the label knowledge of instances in  target domain are unknown, thus, we use support vector machine (SVM) to train a classifier in source domain, and then use this classifier to obtain the pseudo label of instances in  target domain. In our method,  MMD with the linear kernel is used. The corresponding loss is shown in Eq.(\ref{eqMDAadMMMDcon}):
\begin{equation}
MMD^{2}_{con}(\mathcal{Q}_{\ddot{X}_{S}},\mathcal{Q}_{\ddot{X}_{T}})=\sum_{c=1}^{|C|}\bigg\|\frac{1}{n_{s}^{(c)}} \sum_{\widetilde{x_{i}^{s}}\in X_{S}^{(c)}}\widetilde{x_{i}^{s}}W-\frac{1}{n_{t}^{(c)}} \sum_{\widetilde{x_{j}^{t}} \in X_{T}^{(c)}}\widetilde{x_{j}^{t}}W\bigg\|_{F}^{2}= \sum_{c=1}^{|C|}tr(W^{T}\widetilde{X}^{T}M_{c}\widetilde{X}W)
\label{eqMDAadMMMDcon}
\end{equation}
where $|C|$ is the number of classes in source and target domains, and $M_{c}$ is computed as Eq.(\ref{eqmc}):
\begin{equation}
(M_{c})_{ij} = \left\{
             \begin{array}{ll}
             { \frac{1}{n_{s}^{(c)}n_{s}^{(c)}},} & if \quad x_{i},x_{j} \in X_{S}^{(c)} \\
             { \frac{1}{n_{t}^{(c)}n_{t}^{(c)}},} & if \quad x_{i},x_{j} \in X_{T}^{(c)}\\
             {\frac{-1}{n_{s}^{(c)}n_{t}^{(c)}},} & if \quad x_{i}\in X_{S}^{(c)},x_{j} \in X_{T}^{(c)}\\
             {\frac{-1}{n_{t}^{(c)}n_{s}^{(c)}},} & if \quad x_{i}\in X_{T}^{(c)},x_{j} \in X_{S}^{(c)}\\
             {0,} & otherwise
             \end{array}
        \right.
\label{eqmc}
\end{equation}
where $n_{s}^{(c)}$, $n_{t}^{(c)}$ represent the numbers of the $c$-th class instances of source domain and target domain respectively, $X_{S}^{(c)}$, $X_{T}^{(c)}$ are the $c$-th class instances of source domain and target domain respectively.

Then, the  objective function of MDA$_{ad}$ can be formulated as Eq.(\ref{ProveeqARmDAad2}):
\begin{equation}
\mathcal{L}_{MDA_{ad}}(W)=\min\limits_{W}{\frac{1}{K}\sum_{i=1}^{K}\|X-\widetilde{X_{i}}W\|_{F}^{2}} + \lambda\|W\|_{F}^{2} + \beta (tr(W^{T}\widetilde{X}^{T}M_{0}\widetilde{X}W)+\sum_{c=1}^{|C|}tr(W^{T}\widetilde{X}^{T}M_{c}\widetilde{X}W))
\label{ProveeqARmDAad2}
\end{equation}

\textbf{Solution to MDA$_{ad}$:} By setting the derivative of Eq.(\ref{ProveeqARmDAad2}) w.r.t $W$ to zero, the above equation becomes
 \begin{equation}
 \begin{aligned}
\frac {\partial \mathcal{L}_{MDA_{ad}}}{W }= -2\widetilde{X}^{T}X +2\widetilde{X}^{T}\widetilde{X}W +2\lambda W +2\beta \widetilde{X}^{T}(\sum_{c=0}^{|C|}M_{c})\widetilde{X}W=0
\end{aligned}
\label{ProveeqARmDAad3}
\end{equation}

\begin{equation}
 \begin{aligned}
W &= (\widetilde{X}^{T}\widetilde{X}+\lambda I_{d}+\beta\widetilde{X}^{T}(\sum_{c=0}^{|C|}M_{c})\widetilde{X})^{-1}\widetilde{X}^{T}X \\&
=(\mathbb{E}[Q] + \lambda I_{d} + \beta\widetilde{X}^{T}(\sum_{c=0}^{|C|}M_{c})\widetilde{X})^{-1}\mathbb{E}[P]\\&
=(\mathbb{E}[Q] + \lambda I_{d}+ \beta \mathbb{E}[Q_{2}])^{-1}\mathbb{E}[P]
\end{aligned}
\label{ProveeqARmDAad4}
\end{equation}
where $\mathbb{E}[Q_{2}] = \widetilde{X}^{T}(\sum_{c=0}^{|C|}M_{c})\widetilde{X}$.

\begin{equation}
\mathbb{E}[{Q_2}_{ij}] = \left\{
             \begin{array}{lcl}
             { {U_{2}}_{ij}(1-p)^{2}}  &, if \quad i\neq j \\
             {{U_{2}}_{ij}(1-p)}  &,if\quad  i = j
             \end{array}
        \right.
\label{eqmDAQl}
\end{equation}
where $U_{2} = X^{T}(\sum_{c=0}^{|C|}M_{c})X$.
\par In order to capture much more nonlinear relationships in the data, hyperbolic tangent function $tanh(\cdot)$ is used, and  we obtain the global feature representations namely $tanh(\widetilde{X}W)$.
\subsection{Multi-class Marginalized Denoising Autoencoder (MMDA)}
\par It is well known that instances belonging to the same category in two domains have strong similarities. In order to promote cross-domain learning tasks, it is not enough to capture the global information between source and target domains, but also need to learn the invariant feature representations between local subsets of two domains. Therefore, multi-class marginalized denoising autoencoder (MMDA) is proposed to obtain feature representations of the local subsets. More specifically, we first divide the source and target data into  different local subsets according to the category of domains by leveraging the label information of the source domain and the pseudo label information of the target domain.  Then, we learn richer local feature representations of different subsets with MDA. In order to optimize the feature representations, based on the work of MDA, we present multi-class marginalized denoising autoencoder (MMDA) which minimizes the marginal distributions in learning effective local feature representations. The local input data $X^{(1)},X^{(2)}, \ldots, X^{(|C|)}$ are corrupted  with the noise probability $p$. The objective function of MMDA is formalized as Eq.(\ref{ProveeqMMDALOSS}):
\begin{equation}
\mathcal{L}_{MMDA}=\min\limits_{}{\sum_{c=1}^{|C|}(\frac{1}{K}\sum_{i=1}^{K}\|X^{(c)}-\widetilde{X_{i}^{(c)}}W^{(c)}\|_{F}^{2}} + \lambda\|W^{(c)}\|_{F}^{2} + \beta MMD^{2}_{mar}(\mathcal{P}_{\ddot{X}_{S}^{(c)}},\mathcal{P}_{\ddot{X}_{T}^{(c)}}))
\label{ProveeqMMDALOSS}
\end{equation}
where $X^{(c)}=[X_{S}^{(c)},X_{T}^{(c)}]$ and $\widetilde{X_{i}^{(c)}}$ is the $i$-th corrupted version of $X^{(c)}$, and $\ddot{X}_{S}^{(c)}$ and $\ddot{X}_{T}^{(c)}$ are the generated data on a new feature space.
\par \textbf{Solution to MMDA:} By setting the derivative of Eq.(\ref{ProveeqMMDALOSS}) w.r.t $W^{(c)}$ to zero, the above equation becomes
 \begin{equation}
 \begin{aligned}
\frac {\partial \mathcal{L}_{MMDA}}{W^{(c)} }= -2\widetilde{X^{(c)}}^{T}X^{(c)} +2 \lambda W^{(c)} + 2\widetilde{X^{(c)}}^{T}\widetilde{X^{(c)}}W^{(c)}+2\beta \widetilde{X^{(c)}}^{T}M_{0}^{(c)}X^{(c)})\widetilde{X^{(c)}}W^{(c)}=0
\end{aligned}
\label{MMDASOLUTION}
\end{equation}

\begin{equation}
 \begin{aligned}
W^{(c)} &= (\widetilde{X^{(c)}}^{T}\widetilde{X^{(c)}}+\lambda I_{d}^{(c)}+ \beta\widetilde{X^{(c)}}^{T}M_{0}^{(c)}\widetilde{X^{(c)}})^{-1}\widetilde{X^{(c)}}^{T}X^{(c)} \\&
=(\mathbb{E}[Q^{(c)}] +\lambda I_{d}^{(c)}+ \beta\widetilde{X^{(c)}}^{T}M_{0}^{(c)}\widetilde{X^{(c)}})^{-1}\mathbb{E}[P^{(c)}]\\&
=(\mathbb{E}[Q^{(c)}] +\lambda I_{d}^{(c)}+ \beta \mathbb{E}[Q_{2}^{(c)}])^{-1}\mathbb{E}[P^{(c)}]
\end{aligned}
\label{ProveeqWc}
\end{equation}

where $M_{0}^{(c)} = M_{c}$, $I_{d}^{(c)}$ is a diagonal matrix, $\mathbb{E}[Q^{(c)}] = \widetilde{X^{(c)}}^{T}\widetilde{X^{(c)}}$, $\mathbb{E}[Q_{2}^{(c)}] = \widetilde{X^{(c)}}^{T}M_{0}^{(c)}\widetilde{X^{(c)}}$ and $\mathbb{E}[P^{(c)}] = \widetilde{X^{(c)}}^{T}X^{(c)}$.

\begin{equation}
\mathbb{E}[Q_{ij}^{(c)}] = \left\{
             \begin{array}{lcl}
             { U_{ij}^{(c)}(1-p)^{2}}  &, if \quad i\neq j \\
             {U_{ij}^{(c)}(1-p)}  &,if\quad  i = j
             \end{array}
        \right.
\label{eqQC}
\end{equation}

\begin{equation}
\mathbb{E}[{Q_2}_{ij}^{(c)}] = \left\{
             \begin{array}{lcl}
             { {U_{2}}_{ij}^{(c)}(1-p)^{2}}  &, if \quad i\neq j \\
             {{U_{2}}_{ij}^{(c)}(1-p)}  &,if\quad  i = j
             \end{array}
        \right.
\label{eqmQ2C}
\end{equation}

\begin{equation}
  \mathbb{E}[P^{(c)}] =(1-p)U^{(c)}
\label{eqPC}
\end{equation}
where $U^{(c)} = (X^{(c)})^{T}X^{(c)}$, $U_{2}^{(c)} = (X^{(c)})^{T}M_{0}^{(c)}X^{(c)}$. Then, we can obtain the local feature representations $tanh([X_{S}^{(1)}W^{(1)},X_{S}^{(2)}W^{(2)}, \ldots, X_{S}^{(|C|)}W^{(|C|)}])$ and $tanh([X_{T}^{(1)}W^{(1)},X_{T}^{(2)}W^{(2)}, \ldots, X_{T}^{(|C|)}W^{(|C|)}])$.

\par In summary, we present our new semi-supervised domain adaptation algorithm as  \textbf{Algorithm 1}. Following the same strategy adopted by other antoencoders, we use stacked  MDA$_{ad}$ and stacked MMDA to learn deep feature representations. MDA$_{ad}$ and MMDA learn the global and local feature representations layer by layer greedily. Moreover, MDA$_{ad}$ and MMDA  use the same strategy as Chen et al. \cite{17}, MDA$_{ad}$ and MMDA  do not need an end-to-end fine-tuning with the BP algorithm and  can be calculated in a closed form with less time cost. The global and local feature representations learned with stacked MDA$_{ad}$ and stacked MMDA respectively are combined to form dual feature representations. At last, we train a classifier on the source domain with the dual feature representations for classifying  unlabeled data in  target domain. The implementation source code of SSRLDA is available at https://github.com/Minminhfut/SSRLDACode.

\subsection{Generalization bound analysis}
We analyze the generalization error bound of the proposed SSRLDA on the target domain by  the theory of domain adaptation \cite{36,MansourMR09}. We denote that $y(x)$ is the true label function of instance $x$, and $f(x)$ is the label prediction function. $\ddot{X}_{S}$ and $\ddot{X}_{T}$ are the induced data  over the new feature space $\mathcal{Z}$. The expected errors of $f$ with respect to $X_{T}$ and $X_{S}$ are denoted as Eq.(\ref{eqexpectT}) and Eq.(\ref{eqexpectS}) respectively.
\begin{equation}
\epsilon_{T}{(f)}=\mathbb{E}_{x\sim \ddot{X}_{T}}[|y(x)-f(x)|]
\label{eqexpectT}
\end{equation}
\begin{equation}
\epsilon_{S}{(f)}=\mathbb{E}_{x\sim \ddot{X}_{S}}[|y(x)-f(x)|]
\label{eqexpectS}
\end{equation}

\begin{theorem}\label{thm1}
Let $\mathcal{R}$ represent a fixed representation function from $X$ to $\mathcal{Z}$, $\mathcal{H}$ represents the reproducing kernel Hilbert space (a hypothesis space) of $VC$-dimension $d$. If a random labeled instance of size $n_{s}$ is generated by utilizing $\mathcal{R}$ to an $X_{S}$-i.i.d. instance labeled according to $f$, then the expected error of $f$ in $X_{T}$ is bounded with the probability at least 1-$\sigma$ by
\begin{equation}
\epsilon_{T}{(f)} \leq 	\hat{\epsilon_{S}}{(f)} + \sqrt{\frac{4}{n_{s}}(dlog\frac{2en_{s}}{d})+log\frac{4}{\sigma}} + 8Bb MMD_{\mathcal{H}}^{2}(\ddot{X}_{S},\ddot{X}_{T}) + \Theta
\label{eqexpectTTwo}
\end{equation}
where $B$, $b >0$ is a constant, $e$ is the base of natural logarithm, $\hat{\epsilon_{S}}{(f)}$ is the empirical error of $f$ in $X_{S}$,  and $\Theta\ge {\rm inf}_{f\in\mathcal{H}}[\epsilon_{S}{(f)}+\epsilon_{T}{(f)}]$.
\end{theorem}
\begin{proof}
Let  $f^{*} = argmin_{f\in\mathcal{H}}(\epsilon_{S}{(f)}+\epsilon_{T}{(f)})$, and $\Theta_{T}$  be the error of $f^{*}$ with respect to $X_{T}$,   $\Theta_{S }$ be the errors of $f^{*}$ with respect to $X_{S}$, where $\Theta = \Theta_{S } + \Theta_{T}$, $d_{\mathcal{H}}(\ddot{X}_{S},\ddot{X}_{T})=2 \sup\limits_{f\in\mathcal{H}}{|{\rm Pr}_{\ddot{X}_{S}}[f] - {\rm Pr}_{\ddot{X}_{T}}[f]|}$.
\begin{equation}
 \begin{aligned}
\epsilon_{T}{(f)}&\leq \Theta_{T} + {\rm Pr}_{X_{T}}[\mathcal{Z}_{f}\Delta \mathcal{Z}_{f^{*}}]\\
&\leq  \Theta_{T} + {\rm Pr}_{X_{S}}[\mathcal{Z}_{f}\Delta \mathcal{Z}_{f^{*}}] +|{\rm Pr}_{X_{S}}[\mathcal{Z}_{f}\Delta \mathcal{Z}_{f^{*}}]-{\rm Pr}_{X_{T}}[\mathcal{Z}_{f}\Delta \mathcal{Z}_{f^{*}}]|\\
&\leq \Theta_{T} + {\rm Pr}_{X_{S}}[\mathcal{Z}_{f}\Delta \mathcal{Z}_{f^{*}}] + d_{\mathcal{H}}(\ddot{X}_{S},\ddot{X}_{T})\\
&\leq \Theta_{T} + \Theta_{S } + \epsilon_{S}{(f)} + d_{\mathcal{H}}(\ddot{X}_{S},\ddot{X}_{T})\\
&\leq \Theta + \epsilon_{S}{(f)} + d_{\mathcal{H}}(\ddot{X}_{S},\ddot{X}_{T})
\label{Provetarget}
\end{aligned}
\end{equation}
According to the  Vapnik-Chervonenkis theory \cite{Vapnik1998}, the  true $\epsilon_{S}{(f)}$ is bounded by its empirical estimate $\hat{\epsilon_{S}}{(f)}$. If $X_{S}$ contains an $n_{s}$ .i.i.d. instances, then with probability exceeding 1-$\sigma$:
\begin{equation}
\epsilon_{S}{(f)}\leq \hat{\epsilon_{S}}{(f)} + \sqrt{\frac{4}{n_{s}}(dlog\frac{2en_{s}}{d})+log\frac{4}{\sigma}}
\end{equation}
As it can be seen from Eq.(\ref{Provetarget}), the bound relies on the quantity $d_{\mathcal{H}}(\ddot{X}_{S},\ddot{X}_{T})$. One of the main techniques of SSRLDA is that SSRLDA utilizes  MMD distance to bound the discepancy distance $d_{\mathcal{H}}(\ddot{X}_{S},\ddot{X}_{T})$, according to Lemma 1, which is given by Ghifary et al. \cite{GhifaryBKZ17}.
\begin{lemma}[Domain Scatter Bounds Discrepancy \cite{GhifaryBKZ17}]
Let $\mathcal{H}$ be an RKHS with a universal kernel. Consider the hypothesis set
\begin{equation}
 \begin{aligned}
 Hyp = \{f \in \mathcal{H},\quad \|f||_{\mathcal{H}}\leq B \quad and \quad \|f\|_{\infty}\leq b\}
\label{hv}
\end{aligned}
\end{equation}
\end{lemma}
where $B$, $b >0$ is a constant. Let $\ddot{X}_{S}$ and $\ddot{X}_{T}$ be two distributions over $\mathcal{Z}$. Then the following inequality holds:
\begin{equation}
d_{\mathcal{H}}(\ddot{X}_{S},\ddot{X}_{T}) \leq 8Bb MMD_{\mathcal{H}}^{2}(\ddot{X}_{S},\ddot{X}_{T})
\end{equation}
The proof of Lemma 1 can be found in Ghifary et al. \cite{GhifaryBKZ17}.
\end{proof}

In SSRLDA framework, for MDA$_{ad}$, the distribution divergence between source and target domains can be measured by the marginal and condition distribution, and $MMD_{\mathcal{H}}^{2}(\ddot{X}_{S},\ddot{X}_{T}) = MMD^{2}_{mar}(\mathcal{P}_{\ddot{X}_{S}},\mathcal{P}_{\ddot{X}_{T}})+MMD^{2}_{con}(\mathcal{Q}_{\ddot{X}_{S}},\mathcal{Q}_{\ddot{X}_{T}})$. The value of $MMD_{\mathcal{H}}^{2}(\ddot{X}_{S},\ddot{X}_{T})$ is explicitly minimized by distribution adaptation in Eq.(\ref{ProveeqmDAad}). As for MMDA, owning to all instances in  a local subset belonging to the same category, hence only the marginal distribution between two domains is considered, namely $MMD_{\mathcal{H}}^{2}(\ddot{X}_{S},\ddot{X}_{T}) =  \sum_{c=1}^{|C|}MMD^{2}_{mar}(\mathcal{P}_{\ddot{X}_{S}^{(c)}},\mathcal{P}_{\ddot{X}_{T}^{(c)}})$. The value of $MMD_{\mathcal{H}}^{2}(\ddot{X}_{S},\ddot{X}_{T})$ is explicitly minimized  in Eq.(\ref{ProveeqMMDALOSS}). Thus, Theorem 1 indicates the learned feature representations are helpful for obtaining less classify error.

\begin{algorithm}[htb]
\caption{ SSRLDA}
\label{alg:Framwork}
\begin{algorithmic}[1] %这个1 表示每一行都显示数字
\REQUIRE ~~\\ %算法的输入参数：Input
labeled source domain $X_{S}=\{x_{i}^{s}\}_{i=1}^{n_{s}}$ with  label knowledge $Y_{S}=\{y_{i}^{s}\}_{i=1}^{n_{s}}$ ;\\ unlabeled target domain data $X_{T}=\{(x_{j}^{t})\}_{i=1}^{n_{t}}$;\\
number of the layer $l$, noise probability $p$,
hyper-paremeter $\beta$, $\gamma$;
\ENSURE ~~\\ %算法的输出：Output
A classifier $f(\cdot)$ for target domain data ;
\STATE Obtain the target pseudo labels with SVM;
\STATE for $z = 1 : l$ do
\STATE  \quad\quad    Divide the source and target data into different subsets $[X_{S}^{(c)},X_{T}^{(c)}]$ $(c=1,\ldots,|C|)$ \\\quad\quad by using the source labels and target pseudo labels;
\STATE \quad\quad \textbf{MDA$_{ad}$ stage:}
\STATE \quad\quad Initialize $H{1} = []$, $X_{0}=[X_{S},X_{T}]$;
\STATE  \quad\quad    Construct MMD matrix $M_{c}$  $(c=0,\ldots,|C|)$ with Eq.(\ref{eqm0}) and Eq.(\ref{eqmc});
\STATE  \quad\quad    for $k = 1 : z $ do
\STATE  \quad\quad  \quad\quad    Obtain $\widetilde{X_{k-1}}$ by disturbing $X_{k-1}$ with noise probability $p$;
\STATE \quad \quad   \quad\quad   Compute the mapping matrix $W$ with Eq.(\ref{ProveeqARmDAad4});
\STATE  \quad \quad  \quad\quad   Obtain the global feature presentations $X_{k} = [X_{S}(k);X_{T}(k)]=tanh(\widetilde{X_{k-1}}W)$;\\
\STATE  \quad \quad  \quad\quad   Define $H_{1} = [H_{1},X_{k}]$;\\
\STATE \quad \quad end
\STATE \quad\quad \textbf{MMDA stage:}
\STATE \quad\quad Initialize $H_{2} = []$, $X_{0}^{(c)}=[X_{S}^{(c)},X_{T}^{(c)}]$ $(c=1,\ldots,|C|)$;
\STATE  \quad\quad    Construct MMD matrix $M_{0}^{(c)}$  $(c=1,\ldots,|C|)$ with  Eq.(\ref{eqmc});
\STATE  \quad\quad    for $k = 1 : z $ do
\STATE  \quad\quad \quad\quad   Obtain $\widetilde{X_{k-1}^{(c)}}$ by disturbing $X_{k-1}^{(c)}$ with noise probability $p$;
\STATE \quad \quad \quad\quad   Compute the mapping matrix $W^{(c)}$ $(c=1,\ldots,|C|)$ with Eq.(\ref{ProveeqWc});
\STATE  \quad \quad \quad\quad  Obtain the local feature presentations $X_{k}^{(c)} =[X_{S}^{(c)}(k);X_{T}^{(c)}(k)] =tanh(\widetilde{X_{k-1}^{(c)}}W^{(c)})$;\\
\STATE  \quad \quad \quad\quad  Define $H_{2} = [H_{2},[X_{k}^{(1)}; \ldots; X_{k}^{(|C|)}]]$;\\
\STATE \quad \quad end
\STATE \quad \quad    Train a classifier $f(\cdot)$ on the source domain with the feature representations $[H_{1},H_{2}]$;
\STATE \quad \quad    Update the target pseudo labels with the classifier $f(\cdot)$;
\STATE  end
\RETURN classifier $f(\cdot)$; %算法的返回值
\end{algorithmic}
\end{algorithm}
\section{Experiments}
In this section, we perform a comprehensive experimental study on four public datasets to evaluate  the effectiveness  of the proposed SSRLDA model, including 20 Newsgroups  classification, Reuters, Spam and Office-Caltech10 object recognition datasets. Among of them, 20 Newsgroups, Reuters  and Spam datasets are general textual datasets.
\subsection{Data Sets}
\label{datasets}

\textbf{\emph{20 Newsgroups Dataset}}\footnote{http://qwone.com/~jason/20Newsgroups/}.  It contains 18774 news documents with 61188 features. 20 Newsgroups dataset is a hierarchical structure with 6 main categories and 20 subcategories. It contains four largest main categories, including Comp, Rec, Sci, and Talk. The largest subcategory is chosen as the source domain and the second largest subcategory is selected as the target domain for each main category. The task is to classify the main categories and we have adopted a setting as similar to \cite{33}.  In our experiments, the largest category comp is selected as the positive class and one of the three other categories is chosen as the negative class for each setting. The settings of this dataset are listed in Table \ref{20News}. We choose the 5000 most frequent terms as features.

\textbf{\emph{Reuters-21578} (Reuters)}\footnote{http://www.daviddlewis.com/resources/testcollections/}. Reuters is a general dataset for cross-domain text classification tasks.  It consists of three main categories Orgs, People  and Places, and three cross-domain tasks Orgs $\rightarrow$People, Orgs$\rightarrow$Places and People$\rightarrow$Places are constructed.

\textbf{\emph{Spam Dataset}}\footnote{ http://www.ecmlpkdd2006.org/challenge.html}. Spam dataset comes from the ECML/PKDD 2006 discovery challenge. It consists of 4000 labeled training instances, which are obtained from publicly available sources (Public), and half of them are non-spam and the other half are spam.  In addition, the testing instances were obtained from 3 different user inboxes U0, U1 and U2, every user inbox contains 2500 instances. Since the sources are different, the distributions of instances  from source domain and target domain are different. Thus, we construct three cross-domain tasks including Public$\rightarrow$U0, Public$\rightarrow$U1 and Public$\rightarrow$U2.

\textbf{\emph{Office-Caltech10  with SURF features}}.  Office-Caltech10  contains 10 object categories from an office environment in 4 image domains: Amazon (Am), Webcam (We), DSLR (DS), and Caltech256 (Ca) \cite{SunFS16}.  There are 8 to 151 samples per category per domain, and there are 2,533 images in total. Details of the Office-Caltech10  dataset are described in Table \ref{DOC}. In our experiments, due to DSLR containing fewer instances, we only use Amazon (Am), Webcam (We),  and Caltech256 (Ca) for cross-domain learning. In our experiments, six cross-domain tasks  Ca$\rightarrow$AM, Ca$\rightarrow$We, Am$\rightarrow$Ca, Am$\rightarrow$We, We$\rightarrow$Ca, We$\rightarrow$Am are constructed.

\tabcolsep 0.05in
\begin{table}
\centering
\caption{Description of data generated from 20 Newsgroups.}
\label{20News}       % Give a unique label
\begin{tabular}{lp{3cm}p{4cm}}
\toprule
Setting   & Source Domain & Target Domain  \\
\midrule
\multirow{2}{*}{Comp$\rightarrow$Rec }  &  comp.windows.x   rec.sport.hockey  & comp.sys.ibm.pc.hardware  rec.motorcycles \\
\hline
\multirow{2}{*}{Comp$\rightarrow$Sci}   & comp.windows.x   \quad \quad \quad \quad \quad \quad sci.crypt & comp.sys.ibm.pc.hardware  sci.med \\
\hline
\multirow{2}{*}{Comp$\rightarrow$Talk}  & comp.windows.x  talk.politics.mideast & comp.sys.ibm.pc.hardware  talk.politics.guns  \\
\bottomrule
\end{tabular}
\end{table}

\begin{table}
\caption{Description of Office-Caltech10 dataset.}\label{DOC}
\centering
\begin{tabular}{lcccc}
\toprule
Dataset &Type & Instances  &Features & Classes\\\midrule   %第二行线
Caltech-256   &Object &1123 &800 &10\\
AMAZON        &Object &958 &800 &10\\
Webcam        &Object &295 &800 &10\\
DSLR          &Object &157 &800 &10\\\bottomrule   %第三行线
\end{tabular}
\end{table}

\subsection{Baseline Methods}
\label{comparedmethods}
We compare our approach with several state-of-the-art methods to evaluate the effectiveness of our algorithm.
\begin{itemize}
\item \textbf{SVM}\textbf{:} We train a linear SVM classifier on the raw feature space of the labeled source data, and use the classifier   to classify the unlabeled data in target domain.
\item \textbf{CORAL} \cite{SunFS16}\textbf{:} CORAL obtains transferable feature representations  by aligning the second-order statistics of distributions in source and target domains to reduce domain shift.
\item \textbf{DSFT} \cite{WeiKG17}\textbf{:} DSFT constructs a homogeneous feature space by learning the translations between domain specific features and domain common features.
\item \textbf{MDA-TR}\footnote{http://github.com/sclincha/xrce\_msda\_da\_regularization} \citep{20}\textbf{:} MDA-TR learns domain invariant feature representations with an appropriate regularization.
\item \textbf{mSDA}\footnote{http://www.cse.wustl.edu/$\sim$mchen} \citep{17}\textbf{:} mSDA extracts robust feature representations by  marginalization corrupting the raw input data.
\item \textbf{DNFC} \cite{WeiKG16}\textbf{: } DNFC incorporates kernelization and MMD into mSDA to obtain nonlinear deep features. Here we use SVM as the basic classifier.
\item \textbf{$\ell_{2,1}$-SRA} \citep{33}\textbf{:} $\ell_{2,1}$-SRA uses  $\ell_{2,1}$-norm to measure the reconstruction error to reduce the effect of outliers.
\item \textbf{WDGRL}\footnote{https://github.com/RockySJ/WDGRL} \cite{ShenQZY18}\textbf{:} WDGRL is an  adversarial method and it learns domain invariant feature representations by minimizing the wasserstein distance between source domain and target domain.
\end{itemize}

\textbf{Implementation Details:} In our experiments, we set hyper-parameters $\lambda$ = 10, $\beta$ = 0.1  for Spam and Reuters datasets, and set hyper-parameters $\lambda$ = 1E-5, $\beta$ =1000 for 20 Newsgroups dataset.  Layers $l$ is set to 4 and noise probability $p$ is set to 0.9 for 20 Newsgroups  and Spam datasets.  $\lambda$ is set to 1E-5, $\beta$ is set to 1, layer $l$ is set to 3 and noise probability $p$ is set to 0.6 for Office-Caltech10 dataset. For DSFT, we set $\alpha$ = $10^{5}$ and $\beta$ =1.  For MDA-TR, CORAL, DNFC and WDGRL, we use the default parameters as reported in \cite{20}, \cite{SunFS16}, \cite{WeiKG16} and \cite{ShenQZY18} respectively. In the method of mSDA, the best parameters will be shown in the experiment. For $\ell_{2,1}$-SRA, the number of layers is set to 3, 3 and 3 for Spam, 20 Newsgroups and Reuters respectively. The  parameter $\alpha$ is set to 20 for Spam, 5 for 20 Newsgroups and 20 for Reuters respectively. And the parameter $\lambda$ is set to 10, 5 and 30 for Spam, 20 Newsgroups and Reuters respectively. All experimental results are conducted on a PC with Intel(R) i7-7700T, 2.9 GHz CPU, and 16 GB memory.
\par We utilize the classification accuracy which is widely used in the literature \cite{33,20} as the evaluation metric.
\begin{equation*}
Accuracy = \frac{|x:x \in X_{T}\wedge f(x)=y(x)|}{|x:x \in X_{T}|}
\end{equation*}
where $y(x)$ is the true label of instance $x$, and $f(x)$ is the label predicted by the classification model.
\subsection{Experimental Results  of Our Algorithm}
\label{classaccuracy}

\tabcolsep 0.02in

\begin{table}
\caption{Performance (accuracy \%) on Textual datasets.}\label{accuracyTextual}
\centering
\begin{tabular}{lccccccccc}
\toprule
tasks & SVM & CORAL  &DSFT & MDA-TR & mSDA & DNFC &$\ell_{2,1}$-SRA &WDGRL &SSRLDA \\\midrule   %第二行线
Comp$\rightarrow$Rec       &77.13 &70.03 &80.98 &80.32  &80.78 	&84.99  &82.61 	&\textbf{95.66} &93.05\\
Comp$\rightarrow$Sci       &74.77 &72.79 &74.42 &73.86 	&76.45 	&77.06  &79.35 	&83.52 &\textbf{92.52}\\
Comp$\rightarrow$Talk      &92.69 &90.15 &93.27 &94.01 	&93.27  &94.39 	&92.74 	&95.76 &\textbf{97.51}\\\hline
Orgs$\rightarrow$People    &67.88 &70.70 &77.24 &79.06	&77.65 	&79.30  &78.81 	&80.38 &\textbf{98.26}\\
Orgs$\rightarrow$Places    &64.33 &68.17 &69.13 & 71.62 &69.51 	&70.95  &73.15	&75.74 &\textbf{92.33}\\
People$\rightarrow$Places  &53.39 &56.55 &60.17 &60.54  &63.05 	&62.21  &65.83 	&63.70 &\textbf{86.17}\\\hline
Public$\rightarrow$U0      &72.79 &77.03 &73.83 &82.55 	&78.00  &73.76	&81.99  &85.67 &\textbf{91.27} \\
Public$\rightarrow$U1      &73.94 &76.30 &81.14 &85.87 	&85.12 	&67.84  &85.99  &88.26 &\textbf{91.19} \\
Public$\rightarrow$U2      &78.64 &70.60 &78.64 &85.92 	&90.44 	&89.92  &91.36 	&\textbf{95.76} &92.92\\\hline
               Avg         &72.84 &72.48 &77.09 &79.31 	&79.46 	&77.82  &81.31 	&84.94 &\textbf{92.80}\\\bottomrule   %第三行线
\end{tabular}
\end{table}

\tabcolsep 0.038in
\begin{table}
\caption{Performance (accuracy \%) on Office-Caltech10 dataset.}\label{accuracyOfimage}
\centering
\begin{tabular}{lccccccccc}
\toprule
tasks& SVM & CORAL &DSFT  & MDA-TR & mSDA & DNFC & $\ell_{2,1}$-SRA &WDGRL &SSRLDA\\
\midrule   %第二行线
Ca$\rightarrow$Am &51.98 	&50.21  &51.98 &44.15 &54.70 	&52.19  &54.80	 	&55.22 &\textbf{64.82}\\
Ca$\rightarrow$We &34.92 	&33.56 	&34.92 &41.69 &37.29 	&42.37  &37.29	 	&42.37 &\textbf{54.58}\\
Am$\rightarrow$Ca &41.85 	&43.01  &41.85 &37.40 &42.56 	&41.23  &42.83 	 	&45.86 &\textbf{57.79}\\
Am$\rightarrow$We &29.83 	&32.54 	&29.83 &32.88 &34.92 	&37.29  &37.97	 	&\textbf{40.68 }&38.64\\
We$\rightarrow$Ca &28.23 	&29.83  &28.23 &24.93 &31.34 	&31.97  &35.62 	 	&31.08 &\textbf{37.93}\\
We$\rightarrow$Am &30.79 	&30.06 	&30.79 &29.65 &34.66 	&35.39  &38.31	 	&32.15 &\textbf{43.01}\\\hline
Avg               &36.27 	&36.54 	&36.27 &35.12 &39.24    &40.07  &41.14 	    &41.23 &\textbf{49.46}\\
\bottomrule   %第三行线
\end{tabular}
\end{table}

In this section, we compare our SSRLDA with SVM, CORAL, DSFT, MDA-TR, mSDA, DNFC, $\ell_{2,1}$-SRA and WDGRL on  four datasets. Experimental results of four  datasets are reported in Tables  \ref{accuracyTextual} \ref{accuracyOfimage}.  From experimental results, we have the following observations:
\begin{itemize}

\item SSRLDA vs. SVM. According to the values of average accuracy of four datasets, SSRLDA is significantly better than SVM. SVM does not leverage the source  knowledge for transfer learning and results in unsatisfying performance. It shows the effectiveness of our presented domain adaptation framework.
\item SSRLDA vs. CORAL.  SSRLDA outperforms CORAL, and  the classification accuracy improves a lot, especially in the tasks of  Reuters and Spam datasets.   CORAL only learns a  linear transformation to align the second-order statistics of the source and target distributions.  CORAL is a standard shallow learning method, while SSRLDA utilizes deep autoencoder models to learn better feature representations for domain adaptation.   It indicates  the superiority of SSRLDA which uses deep models to obtain complicated feature representations.
\item SSRLDA vs. DSFT.  It can be seen from Tables  \ref{accuracyTextual} \ref{accuracyOfimage}, DSFT  is inferior to SSRLDA.  DSFT is also a   shallow architecture. DSFT utilizes domain common features as a bridge to transfer knowledge for cross-domain learning, it can not well explore the potential relationships among features,  while SSRLDA utilizes deep autoencoder models to obtain  higher feature representations which can  reflect the potential relationship among features. This also demonstrates the effectiveness of SSRLDA.
\item SSRLDA vs. MDA-TR.   SSRLDA performs better than MDA-TR.  MDA-TR only uses single layer autoencoder to learn global feature representations, and it does not take  marginal distribution and condition distribution between source and target domains into account in learning feature representations. While SSRLDA stacks multi-layer autoencoders to obtain complicated feature representations, and  SSRLDA simultaneously minimizes   marginal distribution and condition distribution between two domains in learning feature representations. This shows the superiority of applying multi-layer model and minimizing the distribution discrepancy between two domains.
\item SSRLDA vs. mSDA and $\ell_{2,1}$-SRA.   SSRLDA is superior to mSDA and $\ell_{2,1}$-SRA. Though mSDA  and $\ell_{2,1}$-SRA stack multi-layer  autoencoders to learn global feature representations,  both of them do not consider  the distribution discrepancy between two domains   and do not fully utilize the label information from source domain. While  SSRLDA simultaneously minimizes   marginal distribution and condition distribution between source and target domains, and makes full use of label information in extracting feature representations. SSRLDA achieves the highest accuracy on four datasets. This  demonstrates the superiority of minimizing the distribution between two domains and  incorporating label information.

\item SSRLDA vs. DNFC.   SSRLDA outperforms DNFC.  DNFC minimizes the marginal distribution  between two domains, but it does not consider the condition distribution  in obtaining feature representations. This illustrates  the importance of minimizing the condition  distribution  between source domain and target domain.

\item SSRLDA vs. WDGRL.   WDGRL performs worse than SSRLDA.  WDGRL only learns global feature representations,  while SSRLDA learns both global and local feature representations. SSRLDA combines two types of autoencoders MDA$_{ad}$ and MMDA to learn deep feature representations that are beneficial  to obtain the latent relationships among features. This shows that learning the local feature representations is helpful to cross-domain tasks.

\end{itemize}

\par Overall, SSRLDA  outperforms  state-of-the-art baseline methods on four datasets. More specifically, the average classification accuracies of SSRLDA on Textual and  Office-Caltech10 datasets are 92.80\% and 49.46\% respectively. Compared with the best baseline algorithm WDGRL, the performance was improved by 7.86\% and  8.23\% respectively. Among  these deep learning algorithms, SSRLDA achieves a higher accuracy than MDA-TR,   mSDA, $\ell_{2,1}$-SRA, DNFC, and WDGRL, this indicates learning the local feature representations and incorporating  the label knowledge are able to boost the classification performance, which demonstrates the effectiveness of our proposed SSRLDA.

\begin{figure}
 \centering
\subfigure[20 Newsgroups and Reuters datasets]{
    \begin{minipage}[b]{0.35\textwidth}
    \includegraphics[width=1\textwidth]{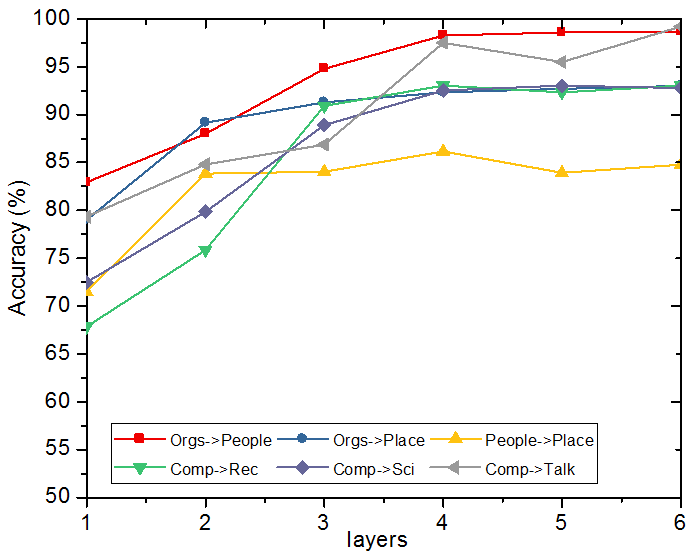}
    \end{minipage}
}
\subfigure[Spam dataset]{
    \begin{minipage}[b]{0.35\textwidth}
    \includegraphics[width=1\textwidth]{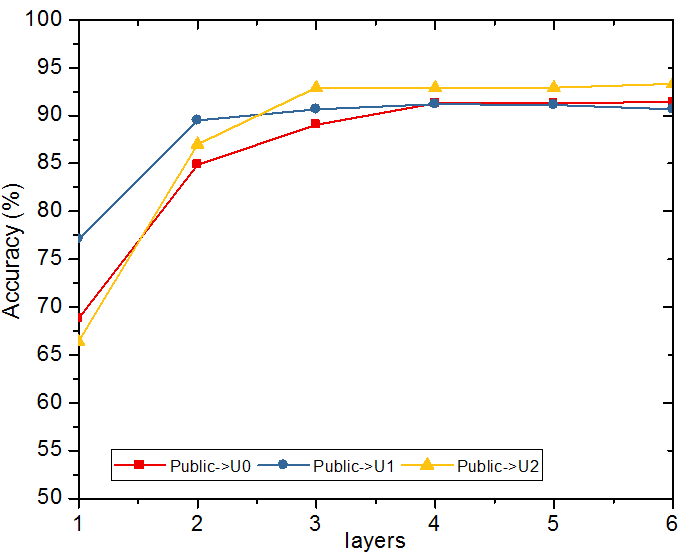}
    \end{minipage}
}
\subfigure[Office-Caltech10 dataset]{
    \begin{minipage}[b]{0.35\textwidth}
    \includegraphics[width=1\textwidth]{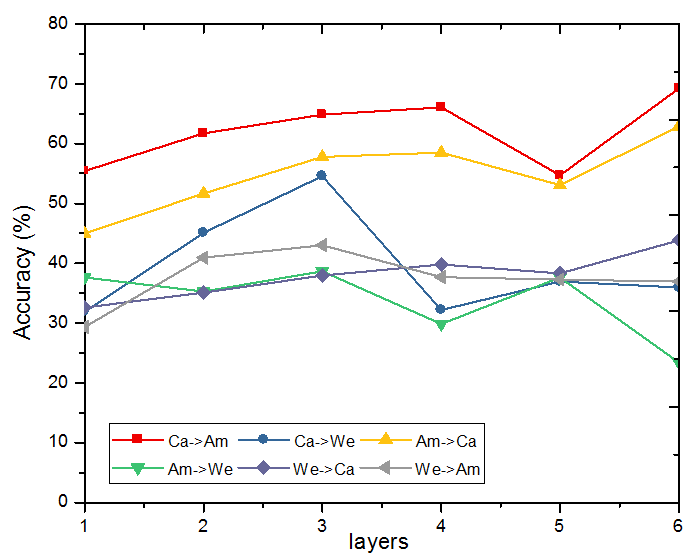}
    \end{minipage}
}
\caption{Accuracy (\%) with different number of layers  $l$ on different datasets.}
\label{parameterlayers}
\end{figure}

\begin{figure}
 \centering
\subfigure[20 Newsgroups and Reuters datasets]{
    \begin{minipage}[b]{0.35\textwidth}
    \includegraphics[width=1\textwidth]{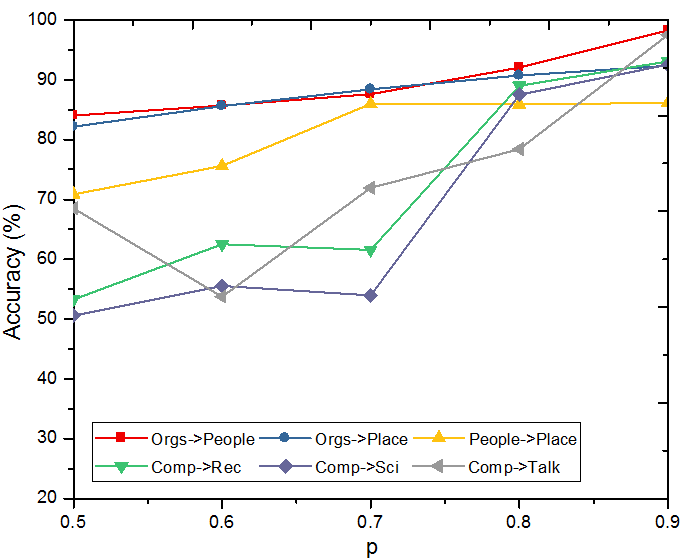}
    \end{minipage}
}
\subfigure[Spam dataset]{
    \begin{minipage}[b]{0.35\textwidth}
    \includegraphics[width=1\textwidth]{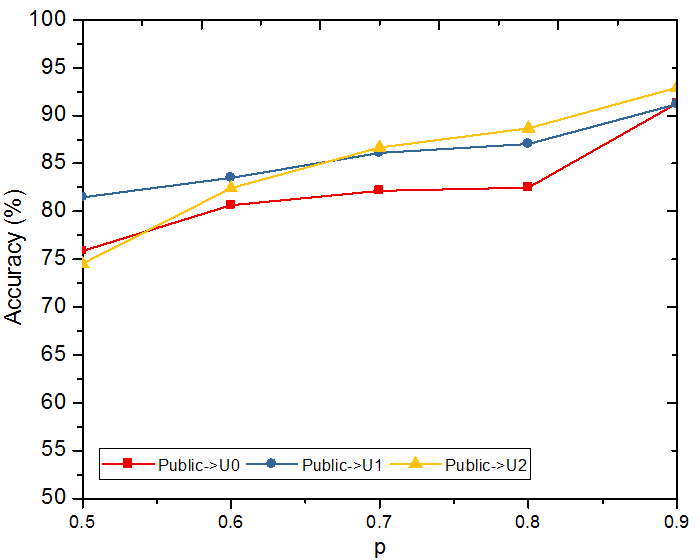}
    \end{minipage}
}
\subfigure[Office-Caltech10 dataset]{
    \begin{minipage}[b]{0.35\textwidth}
    \includegraphics[width=1\textwidth]{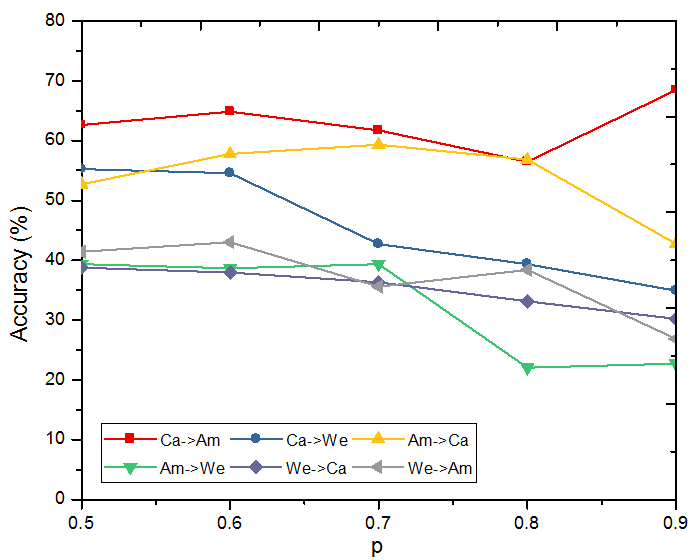}
    \end{minipage}
}
\caption{Accuracy (\%) with different values of $p$ on different datasets.}
\label{parametersp}
\end{figure}

\begin{figure}
 \centering
\subfigure[20 Newsgroups and Reuters datasets]{
    \begin{minipage}[b]{0.35\textwidth}
    \includegraphics[width=1\textwidth]{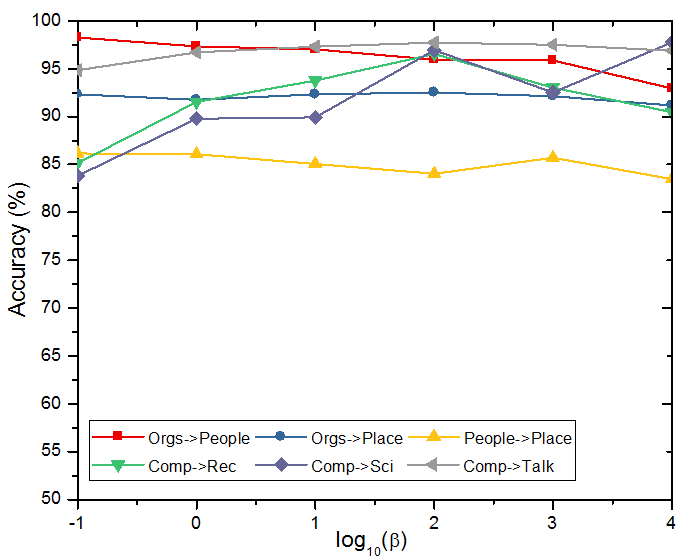}
    \end{minipage}
}
\subfigure[Spam dataset]{
    \begin{minipage}[b]{0.35\textwidth}
    \includegraphics[width=1\textwidth]{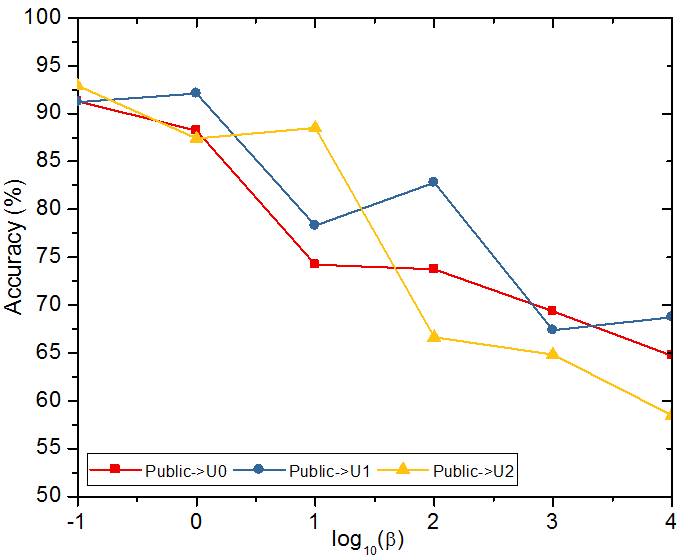}
    \end{minipage}
}
\subfigure[Office-Caltech10 dataset]{
    \begin{minipage}[b]{0.35\textwidth}
    \includegraphics[width=1\textwidth]{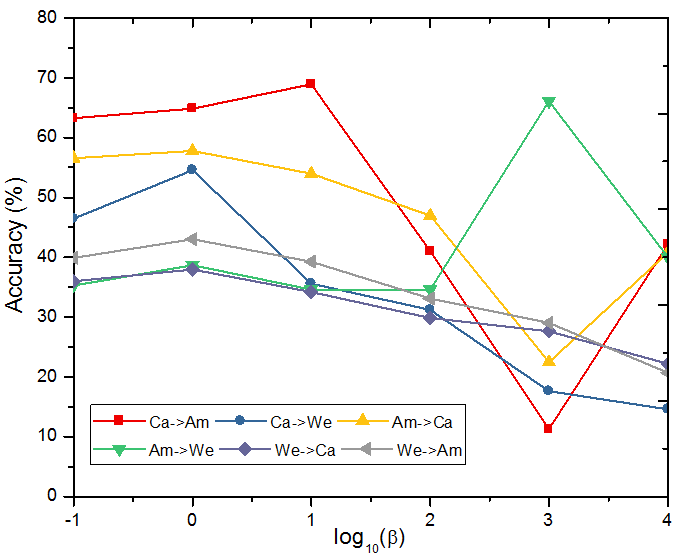}
    \end{minipage}
}

\caption{Accuracy (\%) with different values of $\beta$ on different datasets.}
\label{parametersbeta}
\end{figure}
\subsection{Parameter Sensitivity}
\label{parametersensitivity}
In this section, in order to validate that SSRLDA can achieve an optimal performance under wide range of parameter values, we conduct empirical analysis of  parameter sensitivity. There are four parameters in this approach: the number of layers $l$, noise probability $p$, hyper-parameters $\lambda$ and  $\beta$. We mainly discuss three parameters:  $l$, $p$, $\beta$. In this experiment, when one parameter is tuning, the values of other parameters are fixed.

\textbf{layers $l$:}  We run SSRLDA varying with values of layers $l$ which are  sampled from \{1,2,3,4,5,6\}, and all the results are reported in Fig.\ref{parameterlayers} (a)(b)(c).  On  20 Newsgroups, Reuters and Spam  datasets, the classification accuracy increases with the increasing of stacked layers.  We can achieve satisfying results when the number of stacked layers is set to 4. However, if stacked layers continue to increase, the training time increases a lot but  the classification  accuracy  improves a little. Additionally, we find that  Office-Caltech10 dataset can obtain satisfying results when the number of layers  is set to 3.  Therefore, the number of layers $l$ is set to 4 for 20 Newsgroups, Reuters and Spam datasets, 3 for Office-Caltech10 dataset respectively.

\textbf{noise probability $p$:} We plotted the accuracies with different  values of noise probability $p$ on the four datasets mentioned above, which are shown in Fig.\ref{parametersp} (a)(b)(c). From Fig.\ref{parametersp} (a)(b)(c), we find SSRLDA is sensitive to $p$ and the optimal values of $p$ on different datasets are different. From Fig.\ref{parametersp} (a)(b), we observe that the classification accuracy increases with the increasing of the value of $p$ on textual datasets, $p\in$ [0.8,0.9] is the optimal parameter value for textual  datasets. From Fig.\ref{parametersp} (c),  $p\in$ [0.5,0.6] is the optimal parameter value for Office-Caltech10  datasets.

\textbf{hyper-parameter $\beta$:} We run SSRLDA varying with values of parameter $\beta$ which are sampled from \{0.1,1,10,$10^{2}$,$10^{3}$,$10^{4}$\}, and all the results are shown in Fig.\ref{parametersbeta} (a)(b)(c). From Fig.\ref{parametersbeta} (a)(b)(c), we find that SSRLDA is also sensitive to  $\beta$. We observe that minimizing the marginal distributions between source and target domains  is beneficial to cross-domain learning. According to these observations, $\beta \in$ [$10^{2}$,$10^{3}$] is the optimal parameter value for 20 Newsgroups, while $\beta \in$ [0.1,1] is the optimal parameter value for  the other datasets.
%Reuters, Spam and Office-Caltech10

\begin{figure}
 \centering
\subfigure[20 Newsgroups and Reuters datasets]{
    \begin{minipage}[b]{0.35\textwidth}
    \includegraphics[width=1\textwidth]{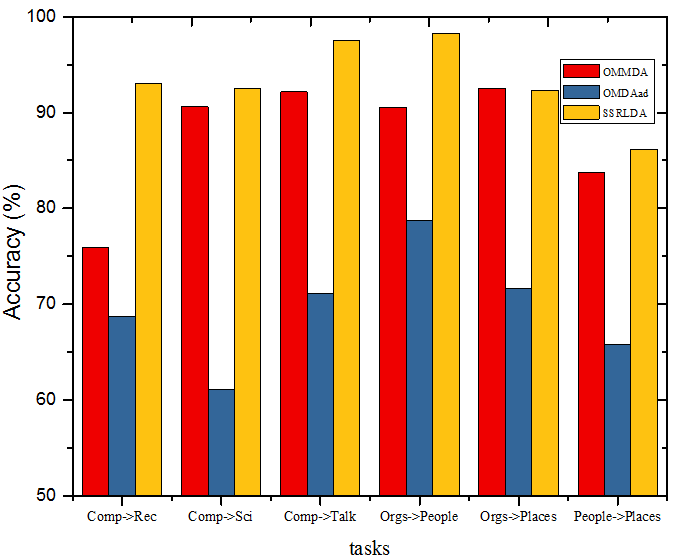}
    \end{minipage}
}
\subfigure[Spam dataset]{
    \begin{minipage}[b]{0.35\textwidth}
    \includegraphics[width=1\textwidth]{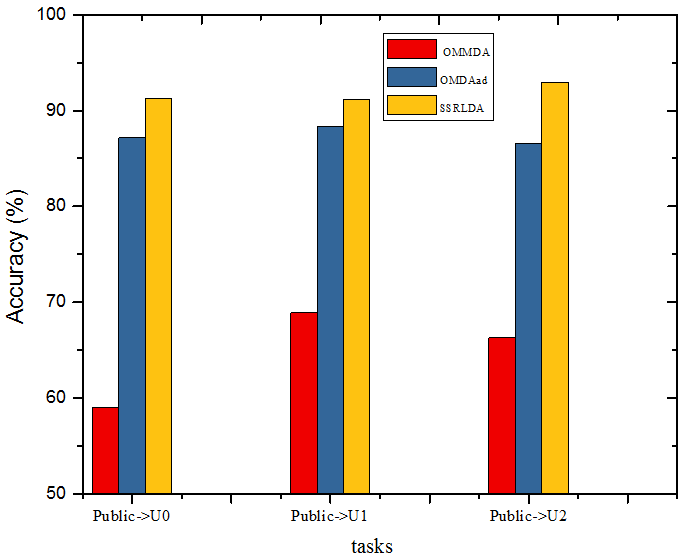}
    \end{minipage}
}
\subfigure[Office-Caltech10 dataset]{
    \begin{minipage}[b]{0.35\textwidth}
    \includegraphics[width=1\textwidth]{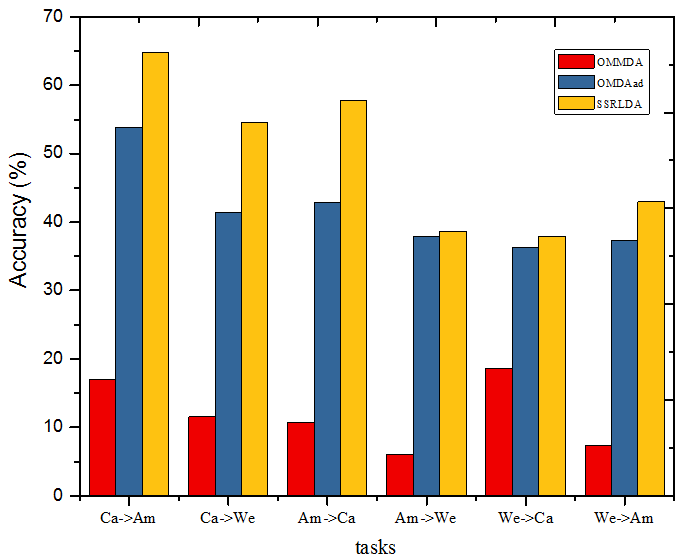}
    \end{minipage}
}
\caption{Comparison of OMDA$_{ad}$, OMMDA and SSRLDA  on different datasets.}
\label{globallocal}
\end{figure}
\begin{figure}
 \centering
\subfigure[20 Newsgroups dataset]{
    \begin{minipage}[b]{0.35\textwidth}
    \includegraphics[width=1\textwidth]{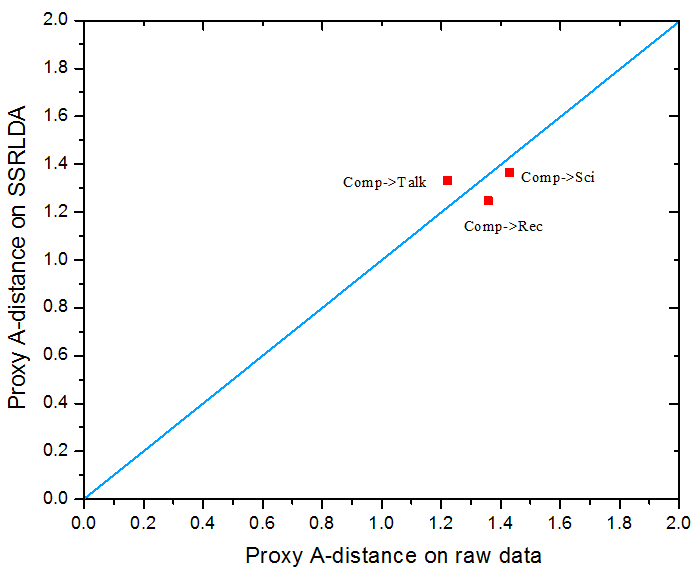}
    \end{minipage}
}
\subfigure[Reuters dataset]{
    \begin{minipage}[b]{0.35\textwidth}
    \includegraphics[width=1\textwidth]{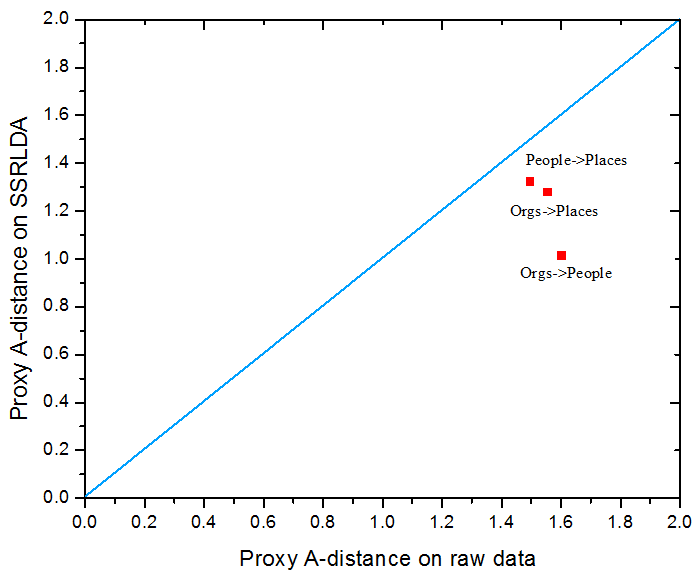}
    \end{minipage}
}
\subfigure[Spam dataset]{
    \begin{minipage}[b]{0.35\textwidth}
    \includegraphics[width=1\textwidth]{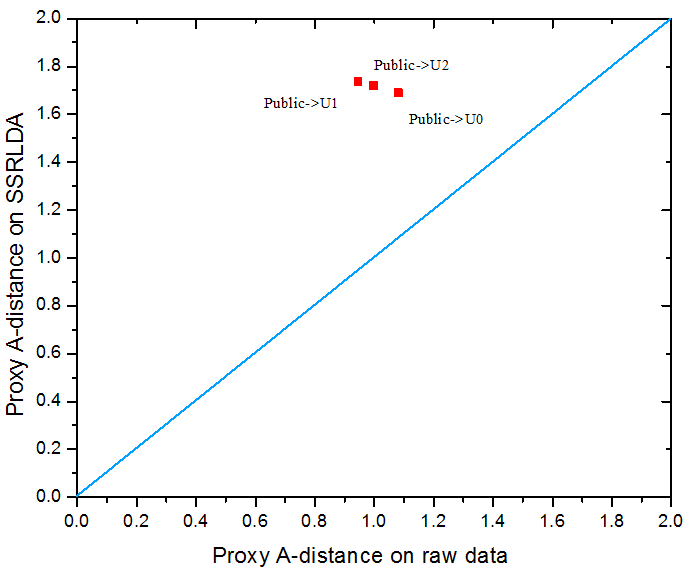}
    \end{minipage}
}
\subfigure[Office-Caltech10 dataset]{
    \begin{minipage}[b]{0.35\textwidth}
    \includegraphics[width=1\textwidth]{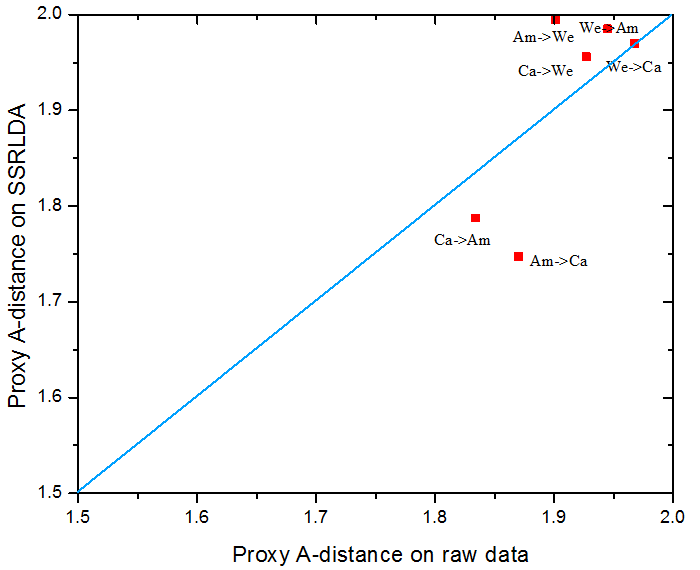}
    \end{minipage}
}
\caption{The Proxy-A-distance  on raw data and SSRLDA.}
\label{Proxy-A-distance}
\end{figure}
\begin{figure}
 \centering
\subfigure[20 Newsgroups dataset]{
    \begin{minipage}[b]{0.35\textwidth}
    \includegraphics[width=1\textwidth]{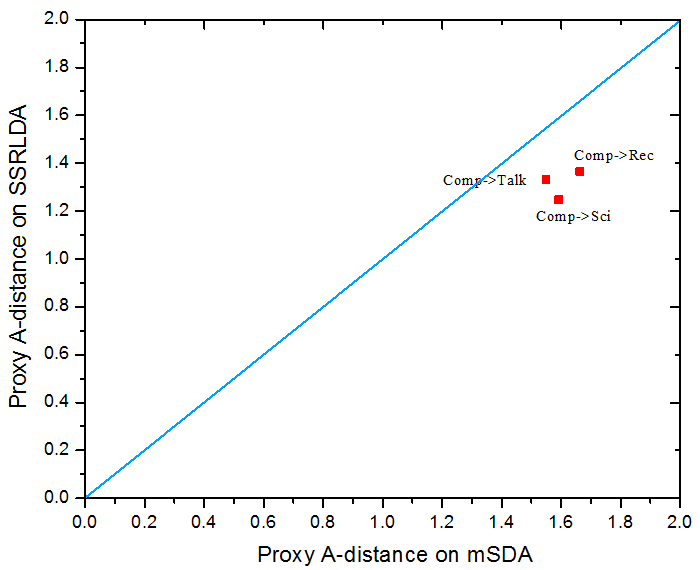}
    \end{minipage}
}
\subfigure[Reuters dataset]{
    \begin{minipage}[b]{0.35\textwidth}
    \includegraphics[width=1\textwidth]{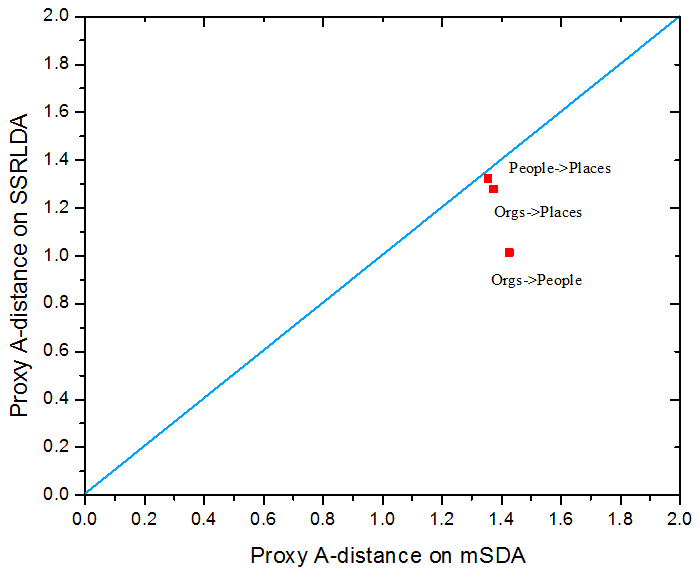}
    \end{minipage}
}
\subfigure[Spam dataset]{
    \begin{minipage}[b]{0.35\textwidth}
    \includegraphics[width=1\textwidth]{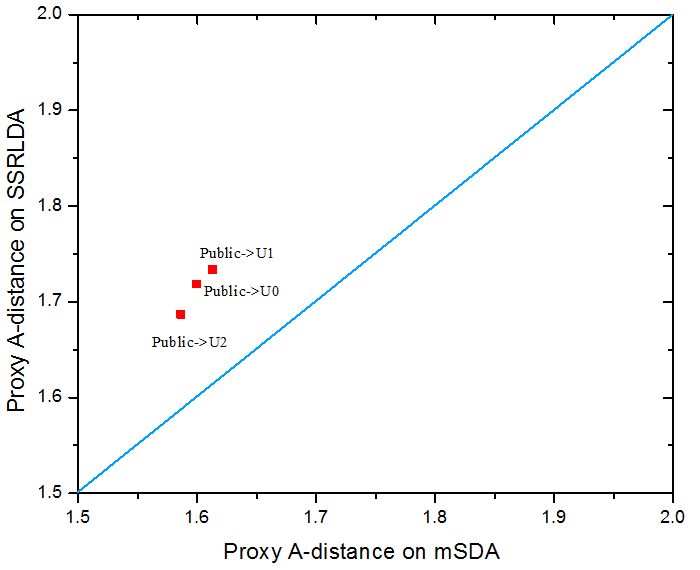}
    \end{minipage}
}
\subfigure[Office-Caltech10 dataset]{
    \begin{minipage}[b]{0.35\textwidth}
    \includegraphics[width=1\textwidth]{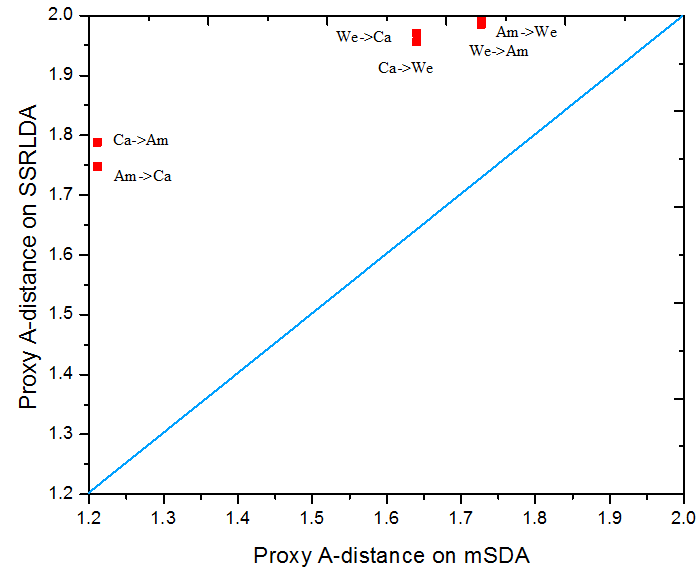}
    \end{minipage}
}
\caption{The Proxy-A-distance on mSDA and SSRLDA.}
\label{Proxy-A-distancemsda}
\end{figure}

\subsection{Effect of Combining  Different Kinds of Autoencoders}
In this section, in order to verify the effectiveness of learning dual feature representations for domain adaptation, we develop two variants of SSRLDA: the one is referred to as OMDA$_{ad}$ (only marginalized denoising autoencoder with  adaptation distributions is used) and the other is referred to as OMMDA (only multi-class marginalized denoising autoencoder is utilized). The proposed SSRLDA is compared with the OMDA$_{ad}$  and OMMDA on  four datasets, and the experimental results are reported in Fig.\ref{globallocal} (a)(b)(c). It can be seen from Fig.\ref{globallocal} (a)(b)(c): 1) we observe that the performance of OMDA$_{ad}$ is inferior to OMMDA on 20 Newsgroups and Reuters, while OMDA$_{ad}$  performs better than OMMDA on  Spam and Office-Caltech10 datasets. The reason  may be that 20 Newsgroups and Reuters are of hierarchical structures, OMMDA can learn  better feature representations of the same hierarchical structure, while  OMDA$_{ad}$  is skilled in learning complicated feature representations on datasets without a hierarchical structure. 2) Overall, SSRLDA is superior to  OMDA$_{ad}$  and OMMDA. It indicates that learning dual feature representations is beneficial to transfer learning rather than only capturing global or local feature representations.  In a word, we can  conclude that SSRLDA can learn more complicated and richer feature representations for cross-domain tasks.

\subsection{Transfer Distance}
A proxy-A-distance $d_{A}=2(1-2\epsilon$) \cite{36} is widely utilized to measure the similarity between two domains which follow different data distributions, where $\epsilon$ is the generalization error of a classifier.  In our experiment, we calculate $\epsilon$ with linear SVM trained on the binary classification problem to discriminate source and target domains.  References \cite{16,17} showed that the proxy-A-distance increases after representations learning, which indicates that the new representations  are helpful for cross-domain learning.  The proxy-A-distance before and after SSRLDA is applied as shown in Fig.\ref{Proxy-A-distance}.  We observe that the proxy-A-distance  increases on Spam dataset and  reduces on Reuters dataset,  while it reduces or increases on different tasks of 20 Newsgroups and Office-Caltech10 datasets, which can be seen from Fig.\ref{Proxy-A-distance} (a)(b)(c)(d). Additionally, due to space limit, we only compare the proxy-A-distance between  SSRLDA and mSDA. From Fig.\ref{Proxy-A-distancemsda} (a)(b)(c)(d),  compared with mSDA, we  find that the proxy-A-distance on SSRLDA increases on Spam and Office-Caltech10 datasets, while it reduces on 20 Newsgroups and  Reuters datasets.  Because SSRLDA achieves satisfying results on the four datasets,  we can obtain the same result as the work in \cite{33}, the proxy-A-distance might become smaller or bigger with new feature representations.

\section{Conclusion and Future Studies}
\label{conclusion}
A novel semi-supervised representation learning approach (SSRLDA) has been proposed for domain adaptation  in this paper. Contrary to existing domain adaptation approaches, our presented approach SSRLDA is the first one to learn global and local feature representations  between source and target domains simultaneously. Additionally, SSRLDA adopts a new strategy to incorporate the label information from source domain to optimize the global and local feature representations. More specifically, SSRLDA  learns better global feature representations by  minimizing  the marginal and condition distributions between source and target domains simultaneously. Furthermore, SSRLDA makes full use of label information to  capture the local information of source and target domains by the means of learning feature representations of instances with the same category in two domains. Our theoretical analysis indicates that SSRLDA guarantees the generalization error bound in target domain.  Finally, extensive experimental results  on textual and image datasets have revealed the effectiveness of   our presented approach.

\par This study  points out that learning local feature representations of instances with the same category and learning global feature representations of two domains are beneficial to  transfer learning. However, SSRLDA only aligns global and local feature representations with an equal importance, while it does not  evaluate the different importance of these two  types of feature representations in cross-domain tasks. In addition, SSRLDA uses the pseudo label knowledge of target domain to minimize the discrepancy across domain, thus the classification performance is susceptible to  the uncertainty of the pseudo-labeling accuracy. In our future work, we will learn an adaptive factor to dynamically leverage the importance of  global and local feature representations, and  use the majority voting method to extract the pseudo label knowledge of target domain. In addition, we plan to improve our work by extending the learning framework of dual feature representations learning framework to that of multiple feature representations.

%conducting more the ground analysis of our work.the dual feature representations
%we plan to improve our work by  conducting more the ground analysis of our work.
\section*{Acknowledgments}
This work is supported in part by the National Key Research and Development Program of China under grant 2016YFB1000901, the Natural Science Foundation of China under grants (61673152,91746209,61876206,61976077) and the Key Laboratory of Data Science and Intelligence Application, Fujian Province University (NO. 1902), and the Anhui Province Key Research and Development Plan (No. 201904a05020073).

\section*{REFERENCES}
\bibliographystyle{elsarticle-num}
\bibliography{refs}
\end{sloppypar}
\end{document}